\documentclass[10pt, journal, twocolumn]{IEEEtran}

\usepackage{graphicx}
\usepackage{amsmath,multirow,booktabs,supertabular}
\usepackage{color}
\usepackage{amssymb}
\usepackage{slashbox}
\usepackage{savesym,url}
\savesymbol{AND}
\usepackage{algorithmic}
\usepackage{algorithm, balance}
\usepackage{multirow}

\newtheorem{theorem}{Theorem}[section]

\newtheorem{proposition}[theorem]{Proposition}
\newtheorem{corollary}[theorem]{Corollary}

\newenvironment{proof}[1][Proof]{\begin{trivlist}
\item[\hskip \labelsep {\bfseries #1}]}{\end{trivlist}}
\newenvironment{definition}[1][Definition]{\begin{trivlist}
\item[\hskip \labelsep {\bfseries #1}]}{\end{trivlist}}

\newcommand{\defeq}{\mathrel{\mathop:}=}

\begin{document}

\title{Optimizing Auto-correlation for Fast Target Search in Large Search Space}
\author{Arif Mahmood, Ajmal Mian and Robyn Owens \\ \hspace{20mm}\\ \textit{This work has been submitted to the IEEE for possible publication. Copyright may be transferred without notice, after which this version may no longer be accessible}
\thanks{School of Computer Science and Software Engineering, The
University of Western Australia, Crawley, WA }} \maketitle

\begin{abstract}

In  remote sensing image-blurring is induced by many sources such as atmospheric scatter, optical aberration, spatial and temporal sensor integration. The natural blurring can be exploited to speed up target search by fast template matching. In this paper, we synthetically induce additional non-uniform blurring to further increase the speed of the matching process. To avoid loss of accuracy, the amount of synthetic blurring is varied spatially over the image according to the underlying content. We extend transitive algorithm for fast template matching by incorporating controlled image blur. To this end we propose an Efficient Group Size (EGS) algorithm which minimizes the number of similarity computations for a particular search image. A larger efficient group size guarantees less computations and more speedup. EGS algorithm is used as a component in our proposed Optimizing auto-correlation (OptA) algorithm.  In OptA a search image is iteratively non-uniformly blurred while ensuring no accuracy degradation at any image location. In each iteration efficient group size and overall computations are estimated by using the proposed EGS algorithm. The OptA algorithm stops when the number of computations cannot be further decreased without accuracy degradation. The proposed algorithm is compared with six existing state of the art exhaustive accuracy techniques using correlation coefficient as the similarity measure. Experiments on satellite and aerial image datasets demonstrate  the effectiveness of the proposed algorithm.

\end{abstract}


\section{Introduction}

Fast template or pattern matching~\cite{GrsPatternMatch2014,GrsImageReg2005,GrsTempMatch1992} is a critical  step in many remote sensing applications such as object detection and recognition~\cite{GrsTargetDet2014,GrsNCCImageReg2010,GrsCraterRecog2007,GrsObjectDet2003,tempObj}, image registration and alignment~\cite{GrsImageReg2014,GrsImageReg2005,Liz1992}, glacier surface movement detection~\cite{GrsGlacierSurfVel2011,GrsGlacier2000}, road detection~\cite{GrsRoadDet2014}, seismic monitoring~\cite{GrsSeismicMonitoring2012}, shadow detection~\cite{GrsShadowDet2014}, cloud detection and tracking~\cite{GrsCloudDet2014}, sea ice tracking~\cite{GrsSeaIceMotTrack2014,GrsSeaIceDrift2014}, stereo image matching for space born imagery~\cite{GrsStereoMatch2007}, DEM generation~\cite{GrsDemGeneration2011}, image super resolution~\cite{GrsNCCImageReg2010}, and content based image retrieval
\cite{GrsImageRet2014,GrsImageRet2013}.  In template
matching, a smaller template or target image is matched at multiple
locations of a larger search image to find the best match location
that maximizes an appropriate similarity measure.

In most of the remote sensing applications, the search space is large which increases the computational complexity of the matching process.  Numerous techniques have been proposed in literature to make the matching
process fast. Based on the search accuracy, these approaches may be
broadly divided into approximate accuracy  and exhaustive accuracy
techniques. The first category obtains fast speedup at the cost of
some loss of accuracy and often incorporates one or more
approximations.  For example, the search space may be approximated
with a smaller search space, the target may be approximated with a
simple representation, or the similarity measure  approximated with a
simpler measure.  The exhaustive accuracy techniques obtain fast
speedup without losing accuracy. This category includes domain
transformation techniques using FFT and bound based computation elimination
algorithms in which unsuitable search locations are skipped from
computations. In this paper we argue on maintaining exhaustive accuracy in the proposed algorithm which skips unsuitable search locations based on bound comparisons. To make the process fast, the search image is processed in a controlled way to avoid any loss of accuracy. 



Bound based computation elimination algorithms are exhaustive accuracy fast template matching techniques. Most of these
algorithms are based on the Sum of Absolute Differences (SAD) or
the Sum of Squared Differences (SSD) 
~\cite{ProjKernels,ouyang2012performance,tombari2009full}.  Elimination algorithms using similarity
measures invariant to the intensity and contrast variations such as Normalized Cross Correlation (NCC) or Correlation
Coefficient ($\rho$) (or Zero-mean NCC) are relatively few including~\cite{EBC,TipPce,TIPTEA}. Algorithms in this category are the main focus of this paper. In most of the remote sensing applications, the images to be matched are acquired at different time of the day, often with significant time lag. Therefore simple measures like SAD and SSD are more vulnerable to errors as compared to the Correlation Coefficient.

Correlation coefficient $(\rho)$ (or ZNCC) is more robust to the linear photo-metric variations between the two images to be matched. Correlation coefficient between template image $t$ and a location in search image $r_o$ is denoted as $\rho_{to}$ and defined as 
\begin{equation}\label{rho1}\rho_{to}=\sum_{x,y} {\frac{(t(x,y)-\mu_t)(r_o(x,y)-\mu_o)}{\sigma_t\sigma_o
}},\end{equation}
$${\sigma_t=\sqrt{\sum_{x,y}
(t(x,y)-\mu_t)^2},\sigma_o=\sqrt{\sum_{
x,y}(r_o(x,y)-\mu_o)^2}},$$ where $\mu_t$ and $\mu_o$ are the
means of $t$ and $r_o$ respectively. If $\{r_c, r_o\} \in \mathcal{R}^{m\times n}$ are  two shifted blocks in the same image, then $\rho_{co}$ in \eqref{rho1} will represent local auto-correlation with a shift $s$. In this paper we use local auto-correlation to speedup the matching process. Local auto-correlation is enhanced by inducing controlled non-uniform image blur in the search images. As a result, we obtain high matching speed with a robust match measure.
  



In the category of correlation coefficient based fast template matching algorithms, we consider  transitive algorithm (TEA)~\cite{TIPTEA} mainly because this algorithm gives the opportunity to incorporate non-uniform blur to obtain high speed template matching. Also TEA has been shown to be faster than the previous algorithms of the same category~\cite{TIPTEA}. However, TEA cannot be directly used for this purpose because the design parameters of this algorithm are currently user defined. Especially the group size parameter and the initial threshold. We develop an Efficient Group Size (EGS) algorithm which maximizes an estimation of the the eliminated computations over group size parameter. We then incorporate EGS algorithm within a novel algorithm for optimization of auto-correlation (OptA) which computes non-uniformly blurred search image.  A scheme for early detection of the high initial threshold is also proposed. All these algorithms combined with TEA make a fast template matching system which has the same accuracy as the original algorithm but obtains significant speedup. Experiments are performed with a  wide variety of the template images on real satellite and aerial image datasets. We observe significant speedup over existing techniques.

\begin{figure}[t] \centering
\includegraphics[width=8.5cm]{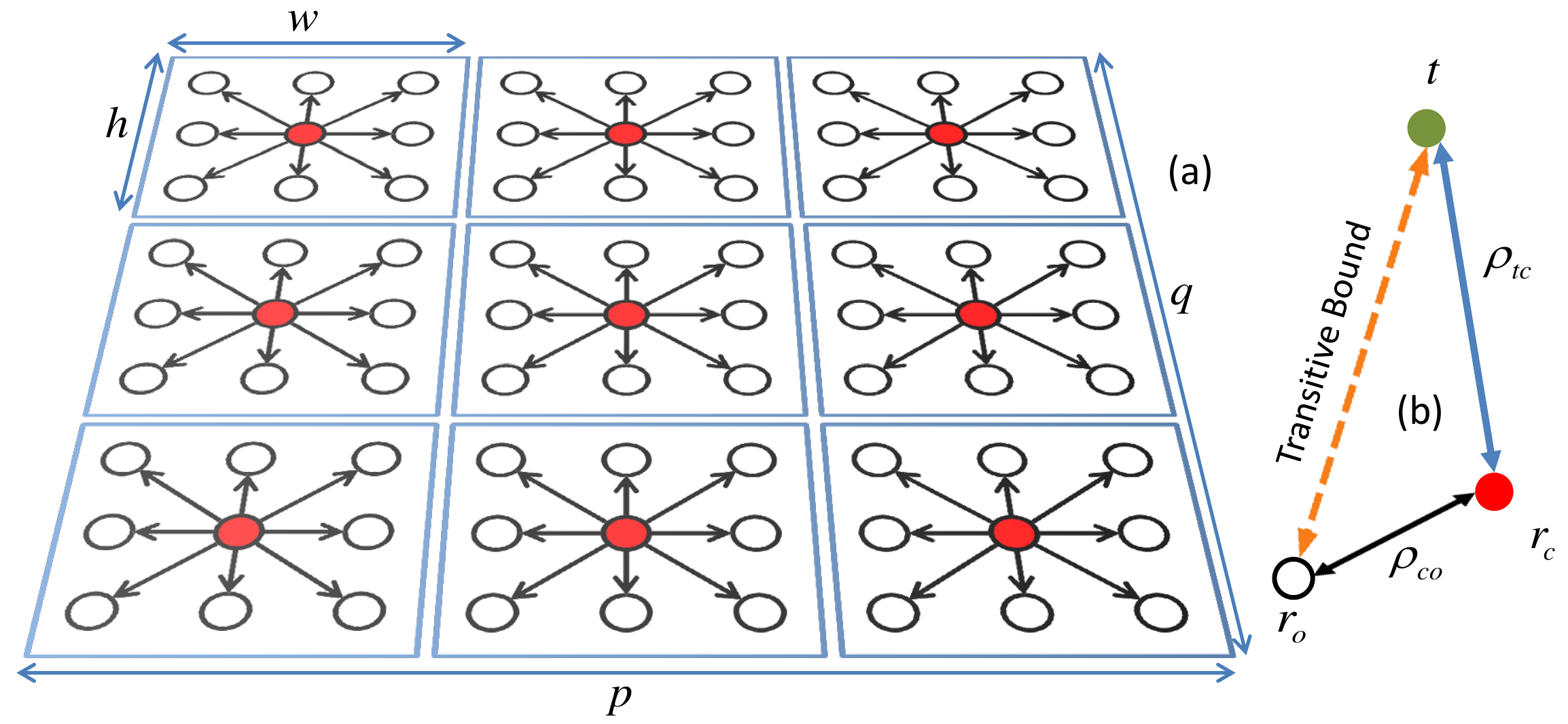}
 \caption{(a) Contiguous search locations in the reference image $I\in \mathcal{R}^{p\times q}$ form groups of size $h\times w$.
 (b) Two known correlations $\rho_{tc}=\rho(t,r_c)$ and $\rho_{co}=\rho(r_c,r_o)$ bound the unknown correlation $\rho_{to}=\rho(t,r_o)$.}
\label{fig:tea2}
\end{figure}

\begin{figure}[t] \centering
\includegraphics[width=8cm]{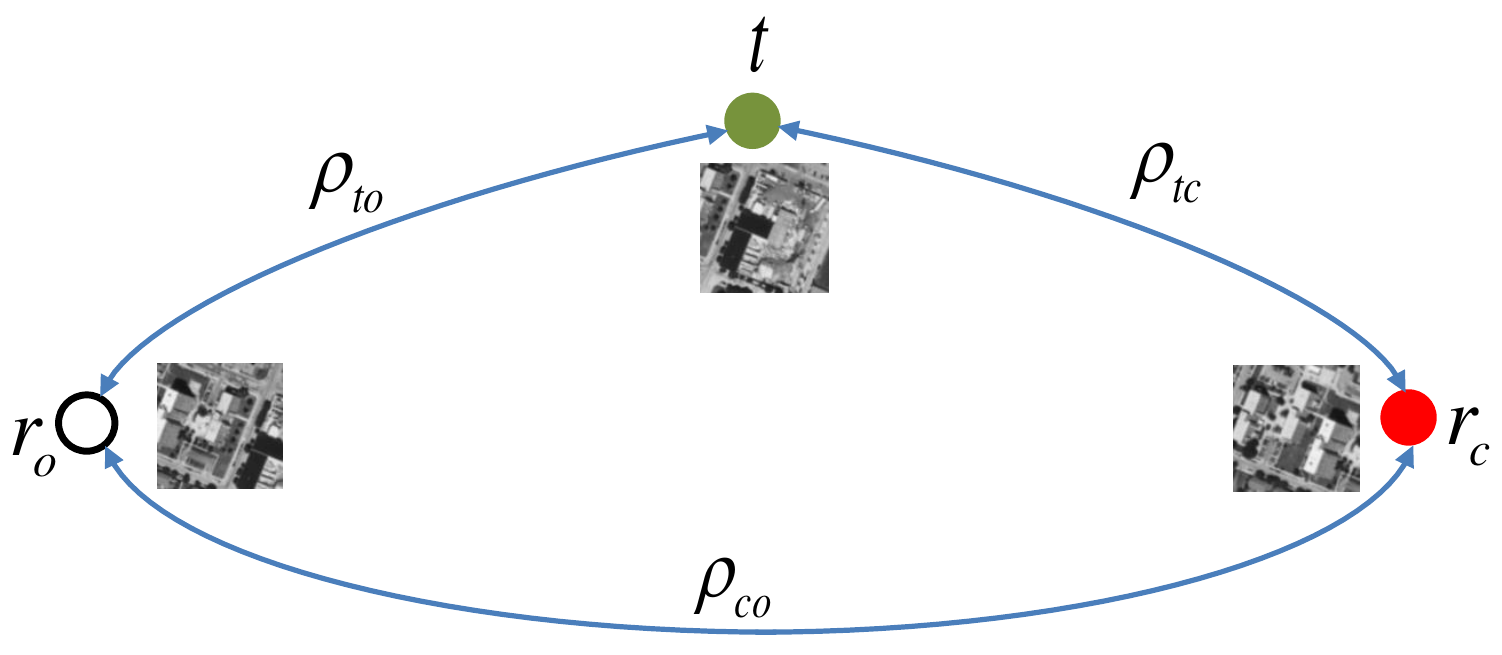}
\vspace{-3mm} \caption{Transitive bounds given by Theorem 1 are
defined on three image blocks \{$t$, $r_o$,
$r_c$\} $\in \mathcal{R}^{m\times n}$ and are quadratic in shape. If
$\rho_{co}=0.81$ and $\rho_{to}=0.61$ then third side is bounded by
$0.017\le \rho_{t,o} \le 0.955$.} \label{fig:threeBlocks}
\end{figure}

\section{A Review of Transitive Elimination Algorithm}
In transitive algorithm the search image is divided into small groups of contiguous search locations. Local auto-correlation of the group center is computed with all other locations within that group and transitive bounds are computed. At run time, the template image is only matched with the group center, while transitive bounds are used to skip all other unlikely locations (Fig. \ref{fig:tea2}). 
 
Transitive  bounds~\cite{TIPTEA}   are defined  for three images, template $t$, outer block $r_o$ and central block $r_c$. Pairwise correlations are given by
\{$\rho_{to}$, $\rho_{tc}$, $\rho_{co}$\} (Fig.
\ref{fig:threeBlocks}). Each correlation is
bounded by the other two correlations by the transitive inequality. Considering $\rho_{to}$ as bounded correlation and the other two as bounding correlations, the upper bound is given by
\begin{equation}\label{tranEqq}\rho_{to} \le  \rho_{tc}\rho_{co}+\sqrt{(1-\rho_{tc}^2)(1-\rho_{co}^2)},\end{equation} and the lower bound is given by \begin{equation} \rho_{to} \ge \rho_{tc} \rho_{co}
-\sqrt{(1-\rho_{tc}^2)(1-\rho_{co}^2)}.\end{equation}
Transitive Gap $\epsilon$ is the difference between the upper and the lower transitive bounds
\begin{equation}\label{eq:eps}\epsilon =2\sqrt{1-\rho_{tc}^2}\sqrt{1-\rho_{co}^2}\end{equation} We observe that $\rho_{to}$ is contained within the transitive gap. The bound tightness can be defined as: the smaller the transitive gap ($\epsilon$), the tighter the transitive bounds. 


\begin{proposition} Transitive gap $\epsilon$ will be minimized if the magnitude of at least one of the two bounding correlations
\{$\rho_{tc}$, $\rho_{co}\}$ is maximized.\end{proposition}
\begin{proof}
Taking the derivative of $\epsilon$ w.r.t any one of the two bounding
correlations and setting it equal to zero ${\partial
\epsilon}/{\partial \rho_{tc}}=0$. From \ref{eq:eps} we get
$$\frac{\partial}{\partial \rho_{tc}}(\rho_{tc}^2\rho_{co}^2
-(\rho_{tc}^2+\rho_{co}^2)+1)=0.$$ $\rho_{tc}(1-\rho_{co}^2)=0$.
Since $\rho_{tc} \ne 0$ therefore $\rho_{co}^2=1$ or $\rho_{co}
\rightarrow \pm 1. $
\end{proof}


In the remote sensing images, consecutive search locations are often highly
correlated, therefore the nearby search locations can be grouped
such that intra-group auto-correlations remain high.  Fig.
\ref{fig:tea2} shows the search image $I$ divided into groups of $h\times w$ search locations. Each group has a central location $(r_c)$ and others are outer locations $(r_o)$. Auto-correlation of each outer location $r_o$ with
the group center $r_c$  is computed by using a very efficient  
Algorithm \ref{algo:LSA}. At run time,  the template image ($t$)
is only matched with $r_c$, while the remaining locations
$r_o$ can be skipped if the sufficient elimination condition given below is satisfied.

\begin{definition}[Sufficient Elimination Condition:]
\label{prop:sec} without loss of accuracy, a search location $r_o\in
I$ can be skipped if there exists another search location $r_b\in I$
such that \begin{equation}\label{eq:SEC}\rho_{tb}\ge
\rho_{tc}\rho_{co}+\sqrt{1-\rho_{tc}^2}\sqrt{1-\rho_{co}^2},\end{equation} where $\rho_{tb}$ is the correlation coefficient \eqref{rho1} between $t$ and $r_b\in \mathcal{R}^{m\times n}$.
\end{definition}
If the sufficient elimination condition is satisfied then it is guaranteed that 
 $ \rho_{tb} \geq \rho_{to}$, therefore $r_o$
cannot exhibit better similarity than $r_b$.

\section{Computing Computational Costs }
In this section we analyze different types of costs involved in the transitive algorithm. These costs include the auto-correlation cost, central correlation cost and the cost of computing correlation on the locations with failed elimination condition. Based on this analysis, we develop an efficient group size algorithm which takes as input the auto-correlation matrix, initial threshold and computes a group size that minimizes the sum of all costs. 

\begin{algorithm}[t]
\caption{Local Auto-correlation Computation Algorithm}
\begin{algorithmic}
 \REQUIRE $I \in \mathcal{R}^{p\times q}$,  $(m,n)$
Template  size, $(h,w)$ Group Size \ENSURE Auto-correlation Matrix
$R_{co}\in \mathcal{R}^{p\times q}$
 \FOR{$i\Leftarrow 1$ to $h$}
     \FOR {$j \Leftarrow 1$ to $w$}
          \FOR { $\forall(x,y)\in I$} 
              \STATE $P(x,y) \Leftarrow I(x,y)\ast I(x+i,y+j)$
          \ENDFOR
          \STATE $S\Leftarrow$ Running-Sum($P,m,n$) 
          \FOR {$\forall (x,y) \in$ Group Centers}
              \STATE $R_{co}(x+i,y+j)\Leftarrow S(x+m,y+n)$
              \STATE $x \Leftarrow x+h$
              \STATE $y \Leftarrow y+w$
          \ENDFOR

 \ENDFOR
 \ENDFOR
 \end{algorithmic}
\label{algo:LSA}
\end{algorithm}

\subsection{Local Auto-correlation Computational Cost ($c_a$)}
Auto-correlation is efficiently computed by algorithm
\ref{algo:LSA}. The reference image is multiplied with its shifted
versions and running sum approach is used to compute the sum of
products over each block requiring only four operations.  The
overall complexity of this algorithm is $O((hw-1)pq)$, where
$h\times w$ is the group size or the number of shifts applied to the
reference image. Note that this cost is significantly smaller than
the cost of a single template matching $O(mnpq)$, because $hw <<
mn$.

\subsection{Central Locations Correlation Cost ($c_c$)}
Central correlation cost $c_c$ is required to match the template
image $t$ with the group centers $r_c$, therefore $c_c$ is proportional to
the number of groups in the search image
\begin{equation}\label{eq:cc}c_c= c_{\rho}\frac{pq}{hw},\end{equation} where $c_{\rho}$ is the one time matching cost.
$c_c$ remains fixed for a particular group size $h\times w$ and can
 be reduced by increasing the group size.

\subsection{Retained Locations Correlation Cost ($c_r$)}
This is the cost of computing the correlation at the 
locations with failed sufficient elimination condition: $c_r=c_\rho n_r$, where $n_r$ are the number of retained locations. As the group
size increases the within group auto-correlation reduces due to
increased distance between $r_c$ and $r_o$. As a result, the transitive
bound becomes loose, causing the sufficient elimination condition to
fail more often and an increase in $c_r$. A relatively precise estimate of retained locations can be found by Proposition \ref{lem:autocor1} while a relaxed but easy to pre-compute estimate can be found by Proposition \ref{prop:simplebound}.

\begin{proposition} \label{lem:autocor1} For a fixed  threshold value $\rho_{tb}$, the sufficient elimination condition will be satisfied on all search locations with  auto-correlation $\rho_{co}$ satisfying the following inequality   $$\rho_{co} \geq \rho_{tb}\rho_{tc} +\sqrt{1+\rho_{tc}^2\rho_{tb}^2-(\rho_{tc}^2+\rho_{tb}^2)}.$$\end{proposition}
\begin{proof} Elimination will be obtained if  
\begin{equation}\label{opt1}\rho_{tb}-\rho_{co}\rho_{tc} \ge \sqrt{1-\rho_{co}^2}\sqrt{1-\rho_{tc}^2},\end{equation}
or simplifying,\begin{equation}\label{opt2} \rho_{co}^2
-2\rho_{tb}\rho_{tc} \rho_{co}+ (\rho_{tc}^2+\rho_{tb}^2-1)\ge
0.\end{equation} Being quadratic,
$\rho_{co}$  has two roots.  
In the practically useful ranges of $\rho_{tc}$ and $\rho_{tb}$, the upper root mostly remains positive while the lower root is mostly negative. Since auto-correlation of the natural images for small lags is positive ($\rho_{co} \ge 0$) therefore in most of the cases only  the upper root is the valid solution. Hence elimination will be obtained if $\rho_{co}$
is larger than the upper root:
\begin{equation}\label{opt3} \rho_{co} \ge
\rho_{tb}\rho_{tc} +
\sqrt{1+\rho_{tc}^2\rho_{tb}^2-(\rho_{tc}^2+\rho_{tb}^2)}.
\end{equation}
\end{proof}


Note that in narrow ranges, lower root may also become positive resulting in more elimination than the estimation based on \eqref{opt3}. However, these cases being less frequent may be ignored from the estimation without inducing significant error.
Proposition \ref{lem:autocor1} can be used to estimate the number of search locations with failed elimination condition if $\rho_{tc}$ is known which limits the beforehand estimation of the retained locations.  However, Proposition \ref{prop:simplebound} enables us to estimate the number of retained locations without knowing $\rho_{tc}$. We empirically observe that the estimation error in Proposition \ref{prop:simplebound} is small and  the result is used to formulate the efficient group size algorithm.

\begin{proposition}\label{prop:simplebound} A search location will be  eliminated if $$E[\rho_{co}] \ge \sqrt{1-\rho_{tb}^2},$$ where the expectation is computed over a  neighborhood around that location.
\end{proposition}
\begin{proof}  Taking expectation of both sides of \eqref{opt3} and considering $\rho_{tb}$ as a fixed value, $$ E[\rho_{co}] \ge
\rho_{tb}E[\rho_{tc}] +
\sqrt{1-\rho_{tb}^2}E\big[\sqrt{1-\rho_{tc}^2}\big].$$ Since $t$ is uncorrelated with most of the search locations in $I$, therefore 
$E[\rho_{tc}]=0$. \begin{equation}\label{cut1} E[\rho_{co}] \ge
\sqrt{1-\rho_{tb}^2}E\big[\sqrt{1-\rho_{tc}^2}\big].\end{equation} Taylor series expansion of $\sqrt{1-\rho_{tc}^2}$ centered at 1 is \begin{equation} \label{taylor}
\sqrt{1-\rho_{tc}^2} \approxeq 1+\frac{\rho_{tc}^2}{2}-\frac{\rho_{tc}^4}{8}.
\end{equation} Since $-1 \le \rho_{tc} \le 1$, therefore $\frac{\rho_{tc}^2}{2}-\frac{\rho_{tc}^4}{8} \ge 0$.  \begin{equation}\label{gt}
E\big[\sqrt{1-\rho_{tc}^2}\big] \ge 1 .\end{equation}
Probability around $\rho_{tc}=0$ is very high and probability around $|\rho_{tc}|=1$ is very low. Therefore, most of the times ${\rho_{tc}^4}/{8}$ will have a negligibly small magnitude. 
$$E\big[\sqrt{1-\rho_{tc}^2}\big] \le 1+\frac{1}{2}E[\rho_{tc}^2] .$$
Note that $\rho_{tc}^2$ is the coefficient of determination~\cite{cramer1987}. If the two sets of numbers have normal distribution, and are uncorrelated the coefficient of determination will have a Beta distribution~\cite{carrodus1992,koerts1969}. The mean of this distribution for the univariate case is $1/(mn-1)$. Therefore $$E\big[\sqrt{1-\rho_{tc}^2}\big] \le 1+\frac{1}{2(mn-1)}.$$ For large $mn$, $\frac{1}{2(mn-1)} \approxeq 0$, therefore \begin{equation}\label{lt}
E\big[\sqrt{1-\rho_{tc}^2}\big] \le 1
\end{equation} From \eqref{gt} and \eqref{lt}, we conclude \begin{equation}\label{et}
E\big[\sqrt{1-\rho_{tc}^2}\big] = 1.
\end{equation}
Therefore  \eqref{cut1} gets simplified to
$$ E[\rho_{co}] \le \sqrt{1-\rho_{tb}^2}$$  
\end{proof}

Using Proposition \ref{prop:simplebound}, an estimate of the correlation cost ($c_r$) on the search locations ($n_r$) where sufficient
elimination condition may fail  is given by 
\begin{equation}\label{eq:co} c_r=c_{\rho} n_r=c_{\rho}\sum_{x=1}^p\sum_{y=1}^q{(\rho_{co}(x,y)\le \sqrt{1-\rho^2_{tb}})},\end{equation} where $(\rho_{co}(x,y)\le \sqrt{1-\rho^2_{tb}})$ will evaluate 1 if true and 0 if false.

\begin{algorithm}[t]
\caption{EGS: Efficient Group Size}
\begin{algorithmic}

 \REQUIRE $I \in \mathcal{R}^{p\times q}$, $(h_0,w_0)$ \{Initial Group Size\},
          $(m,n)$ \{Template Size\}, $\rho_{th} $ \{Threshold\}, $\xi$ \{Cost decrement\}
\ENSURE $(h_{e},w_{e})$, $R_{co}$, $c_{t}$ \{Total estimated Cost\} \STATE
$c_t^{0}\Leftarrow mnpq, k\Leftarrow 1 $, $c_t^1\Leftarrow 0$
  \WHILE{$c_t^{k-1}-c_t^{k}>\xi$} 
    \STATE $c_c^k= {(pq)}/{(h_kw_k)}$
    \STATE $R_{co}^k\Leftarrow$Auto-correlation($I,m,n,h_k,w_k$)
    \FOR {$\forall (x,y)\in I$}
       \IF {$R_{co}^k(x,y)<\sqrt{1-\rho_{tb}^2}$} 
       \STATE $c_r^k \Leftarrow c_r^k+1$
       \ENDIF
     \ENDFOR
      \STATE $c_t^k \Leftarrow c_{\rho}(c_r^k+c_c^k)$
     \IF {$c_t^{k-1}-c_t^{k}> \xi$}
     \STATE $c_t^{k-1} \Leftarrow c_t^k,  h_{k-1}\Leftarrow h_k, w_{k-1}\Leftarrow w_k $
      \STATE $h_{k} \Leftarrow h_{k}+\Delta h $, $w_{k} \Leftarrow w_{k}+\Delta w $
     \ENDIF
     \STATE $k \Leftarrow k+1 $
  \ENDWHILE
 \STATE $c_t\Leftarrow c_t^{k-1}, R_{co}\Leftarrow R_{oc}^{k-1},\{h_e,w_e\}\Leftarrow\{h_{k-1},w_{k-1}\}$
 \end{algorithmic}
\label{algo:EGS}
\end{algorithm}

Total cost $c_t$ is given by the sum of the auto-correlation, central locations and the
retained locations cost: $c_t=c_c+c_r+\frac{c_a}{n_t}$, where $n_t$ is the number of template images to be matched with the same reference image. For large $n_t$, the auto-correlation cost factor $\frac{c_a}{n_t}$ will become very small and can be ignored. Therefore, total cost is given by
\begin{equation}c_t=c_{\rho}(\frac{pq}{h w} +n_r).\end{equation}

\subsection{Computing Efficient Group Size}
We define efficient group size ($w_e,h_e$) as
\begin{equation}\{w_e,h_e\} \defeq \label{objfunc}\min_{h,w} c_{\rho}\left (\frac{pq}{h w} +n_r\right )\end{equation}
\label{sec:costEGS} 
As the group size increases the total cost
given by  \eqref{objfunc} decreases until it hits the minimum
  and then starts increasing. As the gradient of the cost
function changes sign, efficient group size parameters are found. Algorithm \ref{algo:EGS} iteratively solves this optimization problem. In $k^{th}$ iteration  within group auto-correlation matrix $R_{co}^k$ is computed using Algorithm \ref{algo:LSA}. In each iteration,
the computational complexity of Algorithm \ref{algo:LSA} is
$O((h_k w_k-1)pq)$, where ($h_k,w_k$) is the group size in that iteration. Total computational complexity  is of the order of $O(\omega pq)$, where
$\omega=f((h_o+h_{f})(w_o+w_f)/4-1)$,  $f$ is the number of
iterations, ($h_o,w_o$) are the initial and ($h_f,w_f$) are final group sizes.   In order to ensure that the next iteration  only gets executed if the decrease in cost was significant from the last iteration, a parameter $\xi$ is introduced. In our experiments, we fix the value of $\xi$ to 0.5\% of the total cost in the last iteration.

\section{Preventing Loss of Signal Detection Due to Blur}
\label{lossSigDec}

Search image may be blurred to improve auto-correlation ($\rho_{co}$) resulting in larger efficient group size  and hence more speedup. However, uncontrolled blurring may suppress the peaks resulting in the loss of detection rendering the algorithm less accurate than the exhaustive accuracy. We propose to perform controlled non-uniform blurring in different image regions such that speedup is obtained without compromising exhaustive accuracy. 

In order to produce image blur, each pixel  in the image $I$ is
replaced by a weighted average of the pixels in a small neighborhood
around it called  \textit{blur support}. Let $\widehat{I}$ be a
blurred image computed as
\begin{equation}\label{eq:defblur}\widehat{I}(x,y)=\sum_{i=-d}^d\sum_{j=-d}^d{w(i,j)I(x+i,y+j)},\end{equation} where $1 \ge w(i,j)\ge 0$ is a
weight function or spatial averaging filter of size
$(2d+1)\times(2d+1)$ which is the blur support. There are many types of blur filters however the Gaussian averaging mask is most commonly used.
\begin{equation}\label{eq:Gauss}w(i,j)=\frac{1}{\alpha_g} \exp(\frac{i^2+j^2}{2\sigma_g^2}),\end{equation} where
$\alpha_g$ is a normalization term which ensures $\sum {w(i,j)}=1$ and $\sigma_g$ is Gaussian variance which controls the weight distribution and the filter size. In \eqref{eq:Gauss}, the amount of blur depends on the parameter $\sigma_g$ which controls the weight distribution and the filter size parameter $d$ is given by  $\sqrt{-2\sigma_g \ln (t_g)}$, where $t_g$ is the minimum non zero filter value.

The correlation of $t$ with blur image location $\widehat{r}_o \in \widehat{I}$,
represented  by $\widehat{\rho}_{to}$,  is a weighted
average of correlations of $t$ with original image blocks
$r_o\{i,j\}$ over the blur support. As the blur support increases,
correlation peaks with smaller base  suffer more averaging as
compared to the broader peaks. This may result in three types of
accuracy de-gradations including incorrect detection due to
suppression of the correct peak (signal) below a noise peak,
 lack of detection due to suppression of the signal below the initial threshold and  the loss
of localization due to widening of the signal peak. The first type of
degradation may happen if a competitive noise peak with bigger
support is present in the search space.

\subsection{Suppression of the Signal to Noise Ratio}

We argue that for the correct peaks (signal) with support larger than the blur support, the blur process cannot suppress the signal below the noise peaks.

\begin{proposition}\label{pr:order}Suppose there exists two search locations $\{r_a,r_b\}\in I$ such that
$\rho_{at}>\rho_{bt}$. After blur $\widehat{\rho}_{at}
> \widehat{\rho}_{bt}$ is guaranteed if
\begin{equation} \min_{-d \le (i,j) \le d}\{\rho_{at}\{i,j\}\}\ge \max_{-d \le (u,v) \le d}\{\rho_{bt}\{u,v\}\}.\end{equation}
\end{proposition}

\begin{proof} Consider two weight matrices $\{\nu_a,\nu_b\}\in \mathcal{R}^{(2d+1) \times (2d+1)}$
such that  $1 > \{\nu_a(i,j),\nu_b(u,v)\} >0$,
$\sum_{i,j}\nu_a(i,j)=\sum_{u,v}\nu_b(u,v)=1$ and in general
$\nu_a(i,j) \ne \nu_b(u,v)$. By the linearity property
$$\sum_{i,j}{\nu_a(i,j)\rho_{at}\{i,j\}}\ge \sum_{u,v}{\nu_b(u,v)\rho_{bt}\{u,v\}}.$$
We now show that the weight matrices derived from the blurring filters satisfy the sum to 1 property.   If the variances of the blurred images are $\widehat{\sigma}_a$ and $\widehat{\sigma}_b$, the weight matrices will be given by
$$\nu_a(i,j)=\frac{w(i,j)\sigma_a\{i,j\}}{\widehat{\sigma}_a}, \nu_b(u,v)=\frac{w(u,v)\sigma_b\{u,v\}}{\widehat{\sigma}_b}.$$
Since the remotely sensed  images often  have high local auto-correlation, the blurred image variances simplify to
$$\widehat{\sigma}_a \approx \sum_{i,j}{w(i,j)\sigma_a
\{i,j\}} \text{ and } \widehat{\sigma}_b \approx \sum_{u,v}{w(u,v)\sigma_b
\{u,v\}}.$$ 
Substitution of these values yields following weight matrices
\begin{equation}\label{eq:eta1}\nu_a(i,j)=\frac{w(i,j)\sigma_a\{i,j\}}{\sum_{i',j'}{w(i',j')\sigma_a\{i',j'\}}},\end{equation}
\begin{equation}\label{eq:eta2}\nu_b(u,v)=\frac{ w(u,v){\sigma}_b\{u,v\}}
{\sum_{u',v'}{w(u',v')\sigma_b\{u',v'\}}},\end{equation}
where $(i',j')$ are the dummy variables of summation. Therefore $\sum_{i,j}\nu_a(i,j)=\sum_{u,v}\nu_b(u,v)=1$, which proves $\widehat{\rho}_{at}
> \widehat{\rho}_{bt}$.
\end{proof}

\subsection{Loss of Signal Detection}

Let $\rho_{mt}$ be the correlation maximum in the original image and
$\widehat{\rho}_{mt}$ be the corresponding maximum in the blur
image. Loss of signal detection may occur if $\rho_{mt}>\rho_{th}$
but $\widehat{\rho}_{mt}< \rho_{th}$ where $\rho_{th}$ is the
initial correlation threshold.

\begin{corollary}\label{pr:loss1}If
$\rho_{mt}>\rho_{th}$, after blur $\widehat{\rho}_{mt}
> \rho_{th}$ is guaranteed if
\begin{equation} \min_{-d \le (i,j) \le d}\{\rho_{mt}\{i,j\}\}\ge \rho_{th}.\end{equation}
\end{corollary}

\begin{proof} Follows from Proposition \ref{pr:order}.
\end{proof}

\begin{figure} \centering
\includegraphics[width=6.5cm]{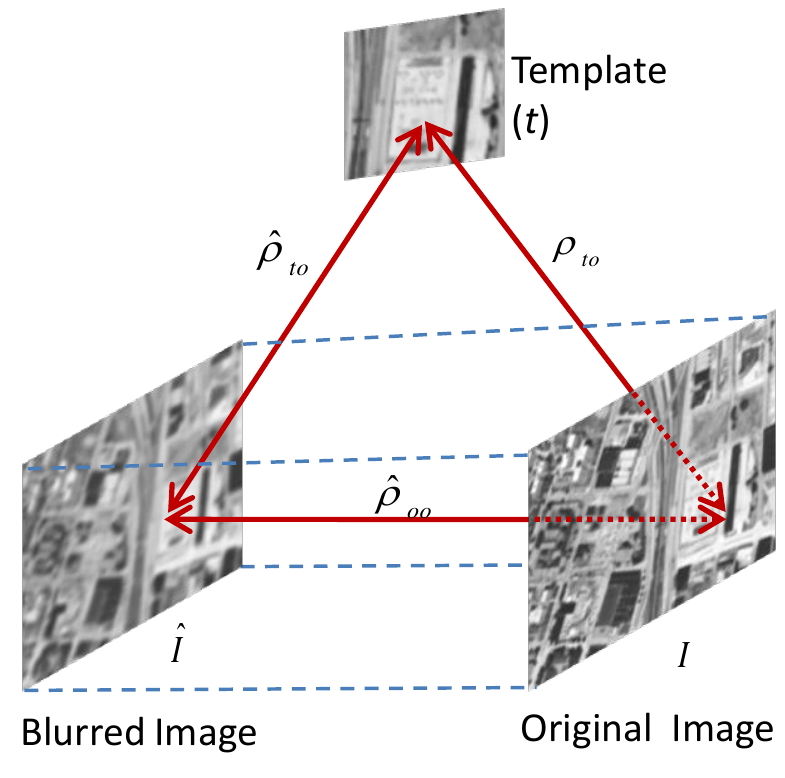}
\vspace{-3mm} \caption{Using transitive bounds for image quality
preservation (Prop. \ref{prop:lsd})} \label{fig:algo2} \vspace{-1mm}
\end{figure}

One may expect a reduction in correlation between $t$ and
$\widehat{r}_o$ due to blurring $\widehat{\rho}_{to}\le \rho_{to}$,
however the reduction can be constrained by using the lower
transitive bound.
\begin{proposition} \label{prop:lsd}If  $\rho_{to}>\rho_{th}$ then $\widehat{\rho}_{to} > \rho_{th}$ is guaranteed if
\begin{equation}\widehat{\rho}_{oo}\ge
\rho_{th}{\rho}_{to}\label{eq:misSig}+\sqrt{1+\rho^2_{th}\rho_{to}^2-(\rho^2_{th}+\rho^2_{to})},\end{equation}
where $\widehat{\rho}_{oo}$ is the correlation coefficient between the original image block 
$r_o$ and its blurred version $\widehat{r}_o$.
\end{proposition}
\begin{proof}
From Fig. \ref{fig:algo2} \begin{equation}\widehat{\rho}_{to} \ge
\widehat{\rho}_{oo}\rho_{to}-
\sqrt{1-\widehat{\rho}_{oo}^2}\sqrt{1-\rho_{to}^2},\end{equation}
 $\widehat{\rho}_{to}>\rho_{th}$ will be guaranteed if
\begin{equation}\label{bound1} \widehat{\rho}_{oo}\rho_{to}- \sqrt{1-\widehat{\rho}_{oo}^2}\sqrt{1-\rho_{to}^2}\ge
\rho_{th}.\end{equation} Squaring both sides and simplifying:
\begin{equation}\label{bonQ}\widehat{\rho}^2_{oo}-2\rho_{th} \widehat{\rho}_{oo}\rho_{to} + (\rho^2_{th}+
\rho^2_{to}-1)\ge 0,\end{equation}  which is quadratic in
$\widehat{\rho}_{oo}$ and only one root  satisfies \eqref{bound1}. For all values of $\widehat{\rho}_{oo} \geq$ this root, Proposition \ref{prop:lsd} will be satisfied.   
\end{proof}

Note that only the suppression of the maximum peak (signal) below the
threshold degrades accuracy. Therefore, putting
$\rho_{to}=\rho_{\max}$ in \eqref{eq:misSig}
\begin{equation}\label{misDec}\widehat{\rho}_{oo}\ge
\rho_{th}{\rho}_{\max}+\sqrt{1+\rho^2_{th}\rho_{\max}^2-(\rho^2_{th}+\rho^2_{\max})}\end{equation}
To avoid miss detection, \eqref{misDec} is ensured to hold for all
search locations.  We propose a very efficient method for the
computation of $\widehat{\rho}_{oo}$ given in algorithm
\ref{algo:blur}. Image locations where inequality \eqref{misDec} is
not satisfied blurring is not applied and the original image
contents are preserved.

\subsection{Loss of Localization}
Loss of localization is recovered by introducing a second matching
stage in which $t$ is matched with $I$ only in a small neighborhood
around the position of $\rho_{max}$ in $\widehat{I}$. We use the
size of this neighborhood the same as the size of the blur support
$(2d+1)\times(2d+1)$.
\begin{figure} \centering
\includegraphics[width=9.1cm]{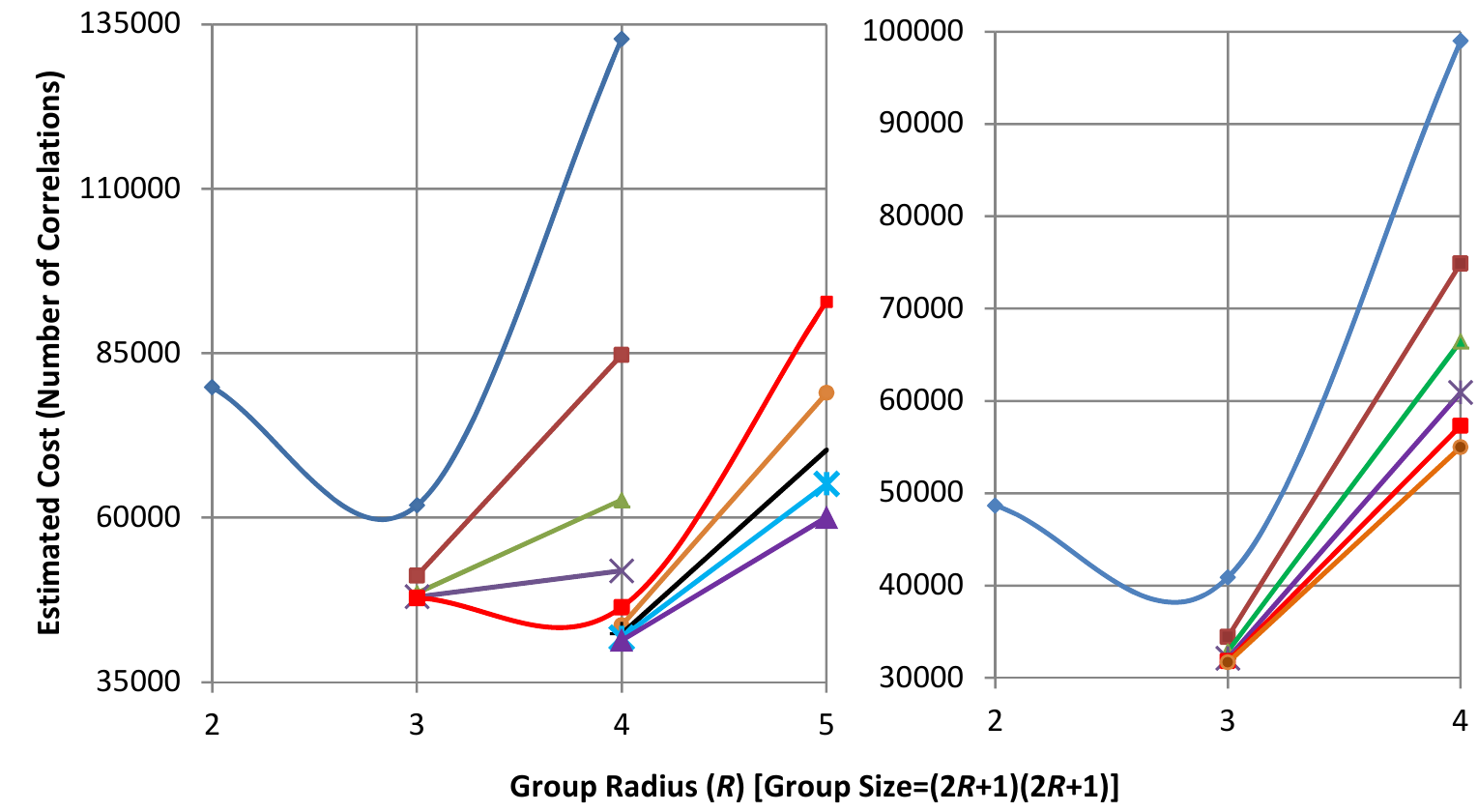}
\caption{ Each curve is an iteration of OptA algorithm and shows variation of total cost ($C_t^k$) with group size by using EGS algorithm. In the consecutive iterations of OptA algorithm, minimum cost given by EGS algorithm decreases. Two different scenarios are shown.  Fig. \ref{fig:VarCost}a: Group size
changed from 7$\times$7 to $9\times 9$ Fig. \ref{fig:VarCost}b: Group size remained 7$\times$7.
To reduce the number of iterations, algorithm is stopped when
$C_t^k-C_t^{k-1}<0.005 C_t^k$.} \label{fig:VarCost}
\end{figure}
\section{Optimizing Auto-correlation (OptA) Algorithm}
\label{sec:optauto}
Auto-correlation between two blurred image blocks increases due to
variance reduction  and also because of the
increased overlap in the blur support. Increased auto-correlation
makes the transitive bounds tight, resulting in more search
locations to be discarded. Algorithm \ref{algo:blur} shows how we
use non-uniform image blur for optimizing auto-correlation to
speedup template matching without accuracy degradation.

In the $k^{th}$ iteration the search image from the last iteration $\widehat{I}_{k-1}$  is again convoluted
with the blur mask $w$ \eqref{eq:Gauss}. Auto-correlation between the
current blurred image $\widehat{I}_{k}$ and the original image $I$ is computed in the matrix $\widehat{R}_{oo}$. All search
locations violating the bound given by \eqref{eq:misSig} are
unblurred by copying back full block contents from the original
image $I$ to the current blurred image $\widehat{I}_{k}$. Efficient group size and the cost are estimated by using
Algorithm \ref{algo:EGS}. If the current cost $c_t^k$ is less than the
previous cost $c_t^{k-1}$ by a relatively big margin $> .005C_t^k$, the algorithm
continues to the next iteration, otherwise we get the optimal blurred image $\widehat{I}_{e}$, 
local auto-correlation matrix $R_{co}^{e}$, and the efficient group size values $(h_{e},w_{e})$.
Fig. \ref{fig:VarCost} shows two plots of cost $c_t^k$ variations within EGS algorithm in consecutive iterations of OptA. Fig.  \ref{fig:VarCost}a shows 9 iterations of OptA algorithm. In the $k=6$ iteration, group size increased from $7\times 7$ to $9 \times 9$. In the later iterations, group size remains fixed, however cost reduced due to tightening of transitive bounds. Fig.  \ref{fig:VarCost}b shows 6 iterations of OptA algorithm. In each iteration, EGS cost reduced and in the final iteration, cost reduction was insignificant.

The blur correlation ($R_{oo}$) computation has a complexity of the
order of $O(pq)$. The complexity of convolution is $O((2d+1)^2pq)$,
where $d$ is the blur support. The dominant computational cost in
each iteration is for the efficient group size algorithm which is
$O(\omega pq)$ as discussed in Section \ref{sec:costEGS}. Therefore,
the overall complexity of Algorithm \ref{algo:blur} is
$O(\kappa\omega pq)$, where $\kappa$ is the number of iterations of
Algorithm \ref{algo:blur}.

\begin{algorithm}[t]
\caption{OptA: Optimizing Auto-correlation using EGS algorithm}
\begin{algorithmic}
\REQUIRE $I\in \mathcal{R}^{p\times q}$, $(m,n)$, $\rho_{th}$,
$\rho_{\max}, (h_o,w_o)$, $W$ \{filter\} \ENSURE $\widehat{I}_{e},
(h_{e},w_{e}), R_{co}^{e}, c_{t}^e$  \STATE
$[h_{1},w_{1},R_{co}^1,c_t^{1}]
\leftarrow\text{EGS}(I,h_o,w_o,m,n,\rho_{th})$\STATE$c_t^{0}
\leftarrow mnpq$ \STATE
$\lambda=\rho_{th}{\rho}_{\max}+\sqrt{1+\rho^2_{th}\rho_{\max}^2-(\rho^2_{th}+\rho^2_{\max})}$

 \STATE ${S}_k=\text{Run-Sum}{({I},m,n)}$\STATE${Q}=\text{Run-Sum}{({I}*{I},m,n)}$
 \STATE $\Omega=\sqrt{{Q}-\frac{1}{mn} {S}*{S}}$
 \STATE $k=1,\widehat{I}_{0}=I$   \COMMENT {$k$ is iteration counter}
\WHILE {$(c_t^{k}<c_t^{k-1})$}
    \STATE $\widehat{I}_k=\widehat{I}_{k-1}\otimes W $ \COMMENT {convolution}

\STATE $P=\widehat{I}_k*{I}$ \COMMENT {$*,\div$ are
point-wise operators}
    \STATE $S_{oo}= \text{Run-Sum}{(P,m,n)}$ 
 \STATE $\widehat{S}_k=\text{Run-Sum}{(\widehat{I}_k,m,n)}$ \STATE$\widehat{Q}_k=\text{Run-Sum}{(\widehat{I}_k*\widehat{I}_k,m,n)}$
    \STATE $\widehat{\Omega}_k=\sqrt{\widehat{Q}_k-\frac{1}{mn} \widehat{S}_k*\widehat{S}_k}$
    \STATE $\widehat{R}_{oo}={(S_{oo}- \frac{1}{mn}\widehat{S}_k *{S})}\div{(\widehat{\Omega}_k*{\Omega})}$
     \FORALL {$(x,y)$}
        \IF {$\widehat{R}_{oo}(x,y) < \lambda$}

           \STATE $\widehat{I}_k(x:x+m,y:y+n)\leftarrow{I}(x:x+m,y:y+n)$ 
           \ENDIF
           \ENDFOR

    \STATE $[h_{k+1},w_{k+1},R_{co}^{k+1},c_t^{k+1}]\leftarrow$ EGS($\widehat{I}_k,h_{k},w_{k},m,n,\rho_{th}$)
    \STATE $k \leftarrow k+1$ 
\ENDWHILE
\STATE $ (\widehat{I}_{e}, h_{e},w_{e}, R_{co}^{e}, c_{t}^e)\leftarrow(\widehat{I}_{k-1}, h_{k-1},w_{k-1}, R_{co}^{k-1}, c_{t}^{k-1})$
 \end{algorithmic}
\label{algo:blur}
\end{algorithm}

\section{Early Detection of the High Threshold}
\label{sec:HighITD} A high initial threshold at the start of the
search process may significantly increase the probability of success
of the sufficient elimination condition and  significantly reduce
the execution time. However, a very high threshold may result in skipping the
best match location while a very small value may result in increased
computational cost.  In the previous version of TEA~\cite{TIPTEA} the initial threshold was a
user defined parameter. We propose a simple strategy to automatically find a suitable threshold value. 
In the previous TEA, the search
space was scanned only once. For each group, the template was
matched with the group center  and  the bounds were computed for the
remaining patches in the group. All patches for which the elimination
test failed were processed before moving to the next group. As a
result, the search space was required to be scanned only once. In
contrast, we propose two scans of the search space. During the first
scan, the template is only matched with the group centers and the
maximum correlation value is tracked. Once all group centers are exhausted,
the maximum value of the central correlation is used as initial
threshold in the second scan.

\begin{table}\centering
 \caption{Size of template images $m\times m$ 
pixels and the number of images $N$ for different datasets shown in Figures
\ref{fig:DataSI} and \ref{fig:DataAI}}\resizebox{9cm}{!}{\small
    \begin{tabular}{c p{5mm}p{5mm}p{5mm}p{5mm}p{5mm}p{5mm}p{5mm}p{5mm}p{5mm}p{5mm}p{5mm}}
    \toprule
    
          &\multicolumn{11}{c}{Satellite Image (SI) dataset}\\
    $m$&21    & 31    & 41    & 51    & 61    & 71    & 81    & 91    & 101   & 111   & 121 \\
    $N$&193   & 197   & 198   & 197   & 198   & 197   & 197   & 197   & 200   & 196   & 198 \\
 \midrule
     &\multicolumn{11}{c}{Aerial Image (AI) Dataset}\\
    
    $m$&29    & 39    & 49    & 59    & 69    & 79    & 89    & 99    & 109   & 119   & 129 \\
    $N$&153   & 175   & 184   & 190   & 198   & 195   & 196   & 199   & 198   & 198   & 200 \\
        \midrule
    \bottomrule
    \end{tabular}}%
  \label{tab:addlabel}%
\end{table}%

\section{Experiments and Results}
In order to  test the proposed fast template matching algorithms, we  have performed extensive experimentation on real satellite and aerial image datasets. Performance of the proposed Efficient Group Size (EGS) algorithm  and the Optimal Auto-correlation (OptA) algorithm (with efficient group size) is separately reported. Both of these algorithms include the proposed strategy for early detection of the high initial threshold. EGS and OptA are compared with six current algorithms including previous TEA~\cite{TIPTEA}, PCE~\cite{TipPce}, FFT~\cite{NumRecipies}, ZEBC~\cite{ZEBC}, AMWU~\cite{AMWU},  SAD based on Successive Elimination  and Partial Distortion Elimination~\cite{PDE2005}. The execution time comparisons are performed on Intel Core i3 CPU \@ 2.10GHz and 3.00GB RAM.

Experiments are performed on  satellite image dataset (SI) of a  low population density seaport (Fig. \ref{fig:DataSI}) and aerial image dataset (AI) of a densely populated area (Fig. \ref{fig:DataAI}).  For each of the two datasets, templates of 11 different sizes are used (Table 1). Total number of templates in the SI dataset are 2168 and in the AI dataset are 2086. The templates are obtained from a different view point at a
different time. Therefore both datasets contain projective distortions as well as illumination variations. The dataset, C++ code and the experimental setup will soon be made publicly available at \url{www.csse.uwa.edu.au/~arifm/OptA.htm}.

\begin{figure} \centering
\includegraphics[width=9cm]{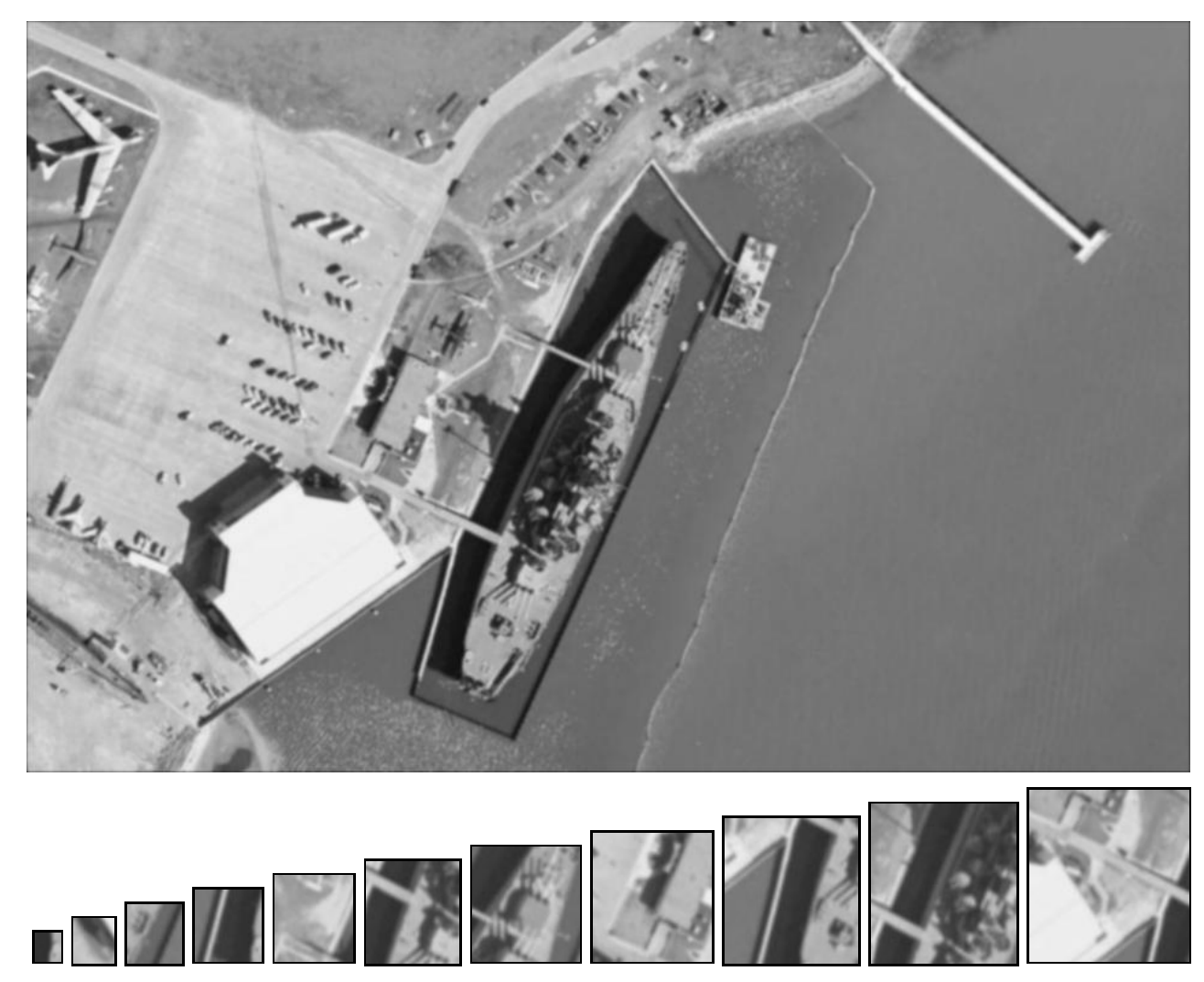}
\caption{Satellite Image (SI) dataset: search image is
962$\times$622 pixels and the template images have sizes varying from {21$\times$21} to {121$\times$121}. } \label{fig:DataSI}
\end{figure}

\begin{figure} \centering
\includegraphics[width=9cm]{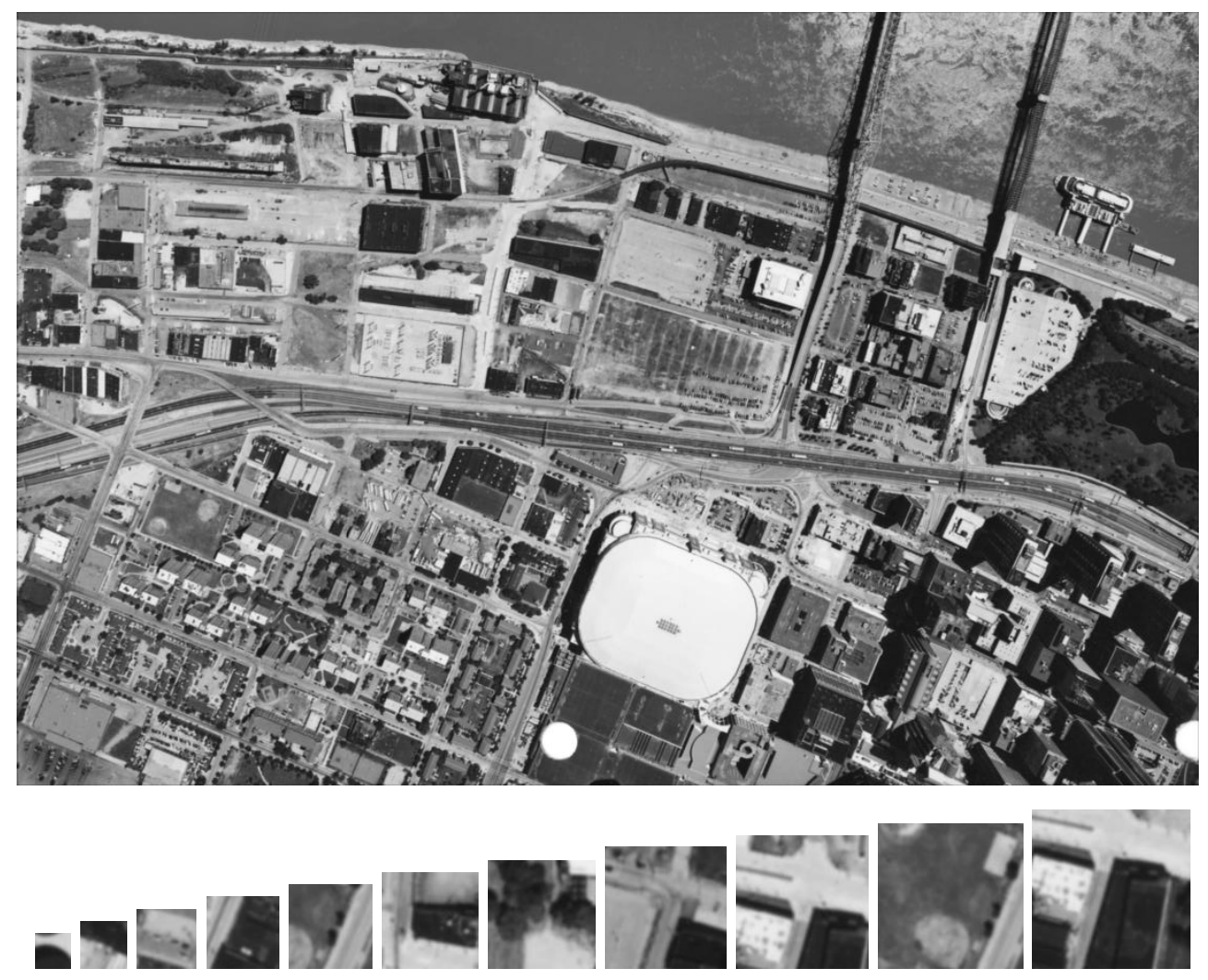}
\caption{ Aerial Image (AI) dataset: search image is
1549$\times$2389 pixels and the template images  vary from {29$\times$29} to
{129$\times$129} pixels.  } \label{fig:DataAI}
\end{figure}


\begin{table}
  \centering
  \caption{Total execution time (sec)  including all related overheads on Satellite Image (SI) dataset. The proposed algorithms
    OptA and EGS are compared  with five existing techniques   }
\resizebox{9cm}{!}{
    \begin{tabular}{c|ccccccc}
    \toprule
          $m$  & OptA   &  EGS   &   TEA  &  PCE   &    ZEBC     &    FFT   &            SAD \\
  \midrule
    21    & 10    & 11    & 13    & 17    & 32    & 128       & 26 \\
    31    & 13    & 15    & 20    & 25    & 114   & 134      & 59 \\
    41    & 17    & 18    & 26    & 42    & 148   & 137      & 96 \\
    51    & 23    & 25    & 38    & 53    & 80    & 133      & 156 \\
    61    & 28    & 31    & 47    & 66    & 253   & 131      & 216 \\
    71    & 38    & 40    & 58    & 80    & 308   & 301      & 303 \\
    81    & 46    & 48    & 65    & 93    & 101   & 311     & 406 \\
    91    & 54    & 57    & 78    & 110   & 123   & 317     & 492 \\
    101   & 67    & 70    & 89    & 127   & 427   & 327     & 596 \\
    111   & 72    & 76    & 97    & 144   & 206   & 337     & 681 \\
    121   & 84    & 87    & 110   & 161   & 175   & 313     & 810 \\
    \bottomrule
    \end{tabular}}%
  \label{tab:SIspeed}%
\end{table}%

\begin{table}
  \centering
  \caption{Total execution time (sec) on Aerial Image (AI) Dataset.
    The proposed algorithms OptA and EGS are compared  with five existing techniques.   }
\resizebox{9cm}{!}{
    \begin{tabular}{c|ccccccc}
    \toprule
          $m$  & OptA   &  EGS   &   TEA  &  PCE   &    ZEBC     &    FFT   &            SAD \\
  \midrule
    29    & 55    & 69    & 76    & 121   & 456   & 1138     & 226 \\
    39    & 82    & 97    & 117   & 234   & 277   & 1288    & 472 \\
    49    & 123   & 139   & 175   & 336   & 277   & 1337    & 798 \\
    59    & 126   & 189   & 237   & 450   & 1203  & 1372    & 1212 \\
    69    & 146   & 238   & 300   & 554   & 250   & 1481    & 1839 \\
    79    & 172   & 300   & 356   & 651   & 1639  & 1442    & 2371 \\
    89    & 205   & 367   & 426   & 795   & 1869  & 1423    & 3108 \\
    99    & 245   & 451   & 500   & 983   & 833   & 1423   & 3960 \\
    109   & 214   & 526   & 580   & 1120  & 2356  & 1436   & 4794 \\
    119   & 225   & 380   & 661   & 1278  & 2575  & 1413   & 5727 \\
    129   & 253   & 422   & 757   & 1470  & 3074  & 1456   & 6822 \\
    \bottomrule
    \end{tabular}}%
  \label{tab:AIspeed}%
\end{table}%


The OptA (Algorithm \ref{algo:blur}) is executed by varying  $W=$\{0.05, 0.20, 0.50,
0.20, 0.05\}, $\rho_{th}=0.80$ and $\rho_{max}=\{1.0, .95, .90\}$. The algorithm is not much sensitive to $\rho_{max}$
because the image regions with more initial blur start failing the
Proposition 4.3 and blurring is stopped in those regions. Therefore,
the three values of $\rho_{max}$ resulted in approximately the same
optimized auto-correlation. The results in Tables \ref{tab:SIspeed} and \ref{tab:AIspeed}   are
reported for $\rho_{max}=0.95$.

In the comparison of OptA with EGS, we observe that OptA has
outperformed EGS in all experiments with a maximum speedup of 2.46
times. No loss of accuracy is observed in the OptA algorithm.
Maximum localization error is observed up to $\pm (4,4)$ pixels.
Both EGS and OptA algorithms outperformed existing TEA~\cite{TIPTEA} (group
size $5\times 5$ and $\rho_th=0.80$) on all datasets (Tables \ref{tab:SIspeed} and \ref{tab:AIspeed}). OptA has obtained maximum 3 times speedup over TEA
and on the average, 2.7\% more elimination on AI dataset (Table
\ref{tab:Elim}).

OptA outperformed PCE~\cite{TipPce} (initial partitions=$m$,
$\rho_{th}=0.80$, two-stage initialization strategy) on all datasets
obtaining maximum speedup of 7.13 times and average computation
elimination up to 9.5\% on AI dataset. This experiment demonstrates significant improvement over current state of the art algorithm.

The OptA  outperformed
ZEBC~\cite{ZEBC} with relatively bigger margins. For ZEBC, the
partition parameter $(r)$ is kept close to 8 and $\rho_{th}=0.80$ as
recommended by~\cite{ZEBC}. Since the template rows must be
divisible by $r$, for each experiment the most efficient value of
$r$ is chosen. Maximum speedup of OptA over ZEBC is 12.13 times. On
SI dataset ZEBC obtained higher computation
elimination. However due to the larger cost of the elimination test,
ZEBC was not able to get the benefit of high computation elimination. We
have also compared OptA with FFT~\cite{NumRecipies} and we observe
up to 22.04 times speed up. Speed up over FFT is generally larger
for smaller template sizes and lesser for the larger sizes. It is
because the computational cost of FFT has relatively less variation
with template size increments.

For the three AI dataset $29\times 29$, $39\times 39$, and $49\times 49$ initial threshold is varied as \{0.70, 0.75, 0.80, 0.85, 0.90\} and the variation of total execution time over these  two datasets for the OptA algorithm is observed as \{294, 287,  287 ,  287,  286\} seconds respectively. The execution time is almost constant for the threshold $\ge$ 0.75. This experiment shows that the sensitivity of the proposed algorithm to the initial threshold is low. 

\begin{table} 
  \centering
  \caption{Average percent computation elimination on the SI and
           AI datasets}
    \begin{tabular}{c|cccccc}
    \toprule
    Dataset  & OptA   &  EGS   &   TEA  &  PCE   &    ZEBC      &      SAD \\
  \midrule
   SI      &    94.5  & 94    & 91.9  & 87.5  & 95.7  & 66.4 \\
   AI       &  96.7  & 95.2  & 94    & 87.2  & 95.2  & 68.3 \\
   \bottomrule
    \end{tabular}%
  \label{tab:Elim}%
\end{table}%

The proposed algorithm (OptA) is also compared with AMWU~\cite{AMWU}.  For this comparison, we prepared two new AI datasets with template sizes $32\times32$ and $64 \times 64$, because AMWU 
requires the template dimensions to be in the powers of 2.   For the AI reference of size 1549$\times$2389, AMWU is more than an order of magnitude slower than our proposed algorithm. AMWU is much efficient on the small search space  $ \le 512 \times 512$~\cite{AMWU}, however for the large search space, its performance degrades probably significantly long lists of the temporary winners  make the hashing scheme inefficient. 

Among the two datasets, OptA obtained more speed up on the AI dataset mainly because it contains  significant details and high frequency contents. OptA was able to induce blurring at most of the AI reference image locations without loss of image quality.  Less speedup on SI dataset is because a big portion of SI reference image is sea, having no details. OptA was not able to blur that portion because blurring caused image quality degradation.

In both datasets ground truth was manually marked and a match was considered correct if it was within $(\pm 4,\pm 4)$ pixels of the ground truth location. Using this criterion, OptA and EGS  exhibited 100\% accuracy on all datasets while SAD has an average accuracy of 81.0\% over AI and
76.11\% over SI  datasets. Maximum speedup of OptA over SAD with  SEA and
PDE is 26.93 times. Although more efficient implementations of SAD 
exist, however accuracy of SAD cannot be improved over the
exhaustive accuracy.

\section{Conclusion}

A fast template matching system was proposed to find a target in a large search space. Fast match speedup was obtained by inducing non uniform  blur in the search space. Exhaustive math accuracy was ensured by maintaining a minimum  image quality at all locations.
For this purpose an Optimal Auto-correlation (OptA) algorithm was proposed. OptA algorithm was based on Efficient Group Size (EGS) algorithm which made the transitive elimination algorithm more efficient. Moreover, a technique for early detection of the high initial threshold was also proposed.  The  fast template matching system was compared with six existing algorithms on two different datasets. The proposed system consistently outperformed the existing techniques.

\section{Acknowledgements}
This research was supported by Australian Research Council grants DP1096801 and DP110102399.


 \balance
 \bibliographystyle{IEEEtran}

\end{document}